\pdfoutput=1
\documentclass{article}
\usepackage{ijcai19}
\usepackage{adjustbox}

\usepackage{tabularx}
\usepackage{subfiles}
\usepackage{times}
\usepackage{soul}
\usepackage{url}
\usepackage[utf8]{inputenc}
\usepackage[small]{caption}
\usepackage{graphicx}
\usepackage{amsmath}
\usepackage{booktabs}
\usepackage{algorithm}
\usepackage{algorithmic}
\usepackage{subfig}
\usepackage{adjustbox}
\usepackage[shortlabels]{enumitem}
\title{Can Meta-Interpretive Learning outperform Deep Reinforcement Learning of Evaluable Game strategies?}

\author{
C\'{e}line Hocquette$^1$
\and
Stephen H. Muggleton$^1$
\affiliations
$^1$Department of Computing, Imperial College London, London, UK\\
\emails
\{celine.hocquette16, s.muggleton\}@imperial.ac.uk
}

\begin{document}
\maketitle
\begin{abstract}
World-class human players have been outperformed in a number of complex two person games (Go, Chess, Checkers) by Deep Reinforcement Learning systems. However,
owing to tractability considerations minimax regret of a learning system cannot be evaluated in such games. In this paper we consider simple games (Noughts-and-Crosses and Hexapawn) in which minimax regret can be efficiently evaluated. We use these games to compare Cumulative Minimax Regret for variants of both standard and deep reinforcement learning against two variants of a new Meta-Interpretive Learning system called \textit{MIGO}. In our experiments all tested variants of both normal and deep reinforcement learning have worse performance (higher cumulative minimax regret) than both variants of \textit{MIGO} on Noughts-and-Crosses and Hexapawn. Additionally, \textit{MIGO}'s learned rules are relatively easy to comprehend, and are demonstrated to achieve significant transfer learning in both directions between Noughts-and-Crosses and Hexapawn.
\end{abstract}
\section{Introduction}
Deep Reinforcement Learning systems have been demonstrated capable
of mastering two-player games such as Go \cite{GO}, outperforming
the strongest human players. However, these systems 1) generally require
a very large training set to converge toward a good strategy, 2) are
not easily interpretable as they provide limited explanation about
how decisions are made and 3) do not provide transferability of the learned
strategies to other games.

We demonstrate in this work how machine learning strategies as logic
programs can overcome these limitations. For example, an applicable
strategy for playing Noughts-and-Crosses is to create double attacks
when possible. An example of this is shown in Figure \ref{fig:intro}.
Player O executes a move from board A to board B which creates the two threats represented in green, and results in a forced win for O.
The rules in Figure \ref{fig:intro} describe such a strategy.
A,B and C are variables representing state descriptions which encode
the board together with the active player. The rules state that
a move by the active player from A to B  is a winning
move if the opponent cannot immediately win and the opponent cannot
make a move to prevent an immediate win by the active player. These rules
provide an understandable strategy for winning in two moves. Moreover, these
rules are transferable to more complex games as they are generally true for
describing double attacks.
\begin{figure}[t]
\centering
\captionsetup{justification=centering}
\subfloat{{\includegraphics[width=0.18\textwidth]{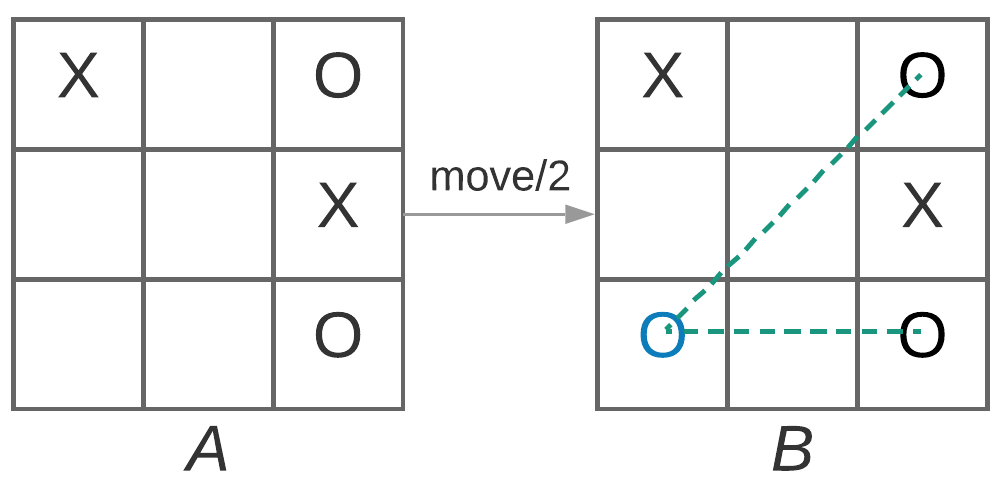}}} \hfill
\subfloat{
\begin{adjustbox}{max width=0.35\textwidth}
\begin{tabular}{|c|}
\hline
\texttt{win\_2(A,B):-win\_2\_1\_1(A,B),not(win\_2\_1\_1(B,C)).}\\
\texttt{win\_2\_1\_1(A,B):-move(A,B),not(win\_1(B,C)).}\\
\texttt{win\_1(A,B):- move(A,B),won(B).} \\
\hline
\end{tabular}
\end{adjustbox}}
    \caption{Noughts and Crosses: example of optimal move for O from board A to board B. For all moves of X from board B, O can win in one move. This statement can be expressed with the logic program presented: O makes a move such that X cannot immediately win nor make a move that blocks O.}
    \label{fig:intro}
        \vspace{-3mm}
\end{figure}
We introduce in this article a new logical system called \textit{MIGO}
(\textit{Meta-Interpretive Game Ordinator})\footnote{From the children's
game-playing phrase {\em My go!} and the literal translation
into English of the French word \textit{Ordinateur} which means computer.}
designed for learning two player game optimal strategies of the form
presented in Figure \ref{fig:intro}. It benefits from a strong inductive
bias which provides the capability to learn efficiently from a few
examples of games played. Learned hypotheses are provided in a symbolic
form, which allows their interpretation. Moreover, learned strategies
are generally true for all two-player games, which provides straightforward
transferability to more complex games.

\textit{MIGO} uses Meta-Interpretive Learning (MIL), a recently developed
Inductive Logic Programming (ILP) framework that supports predicate
invention, the learning of recursive programs \cite{[MIL1],[MIL]}
and Abstraction \cite{[Abstraction]}. \textit{MIGO}
additionally supports Dependent Learning \cite{[Onsshot]}. The learning
operates in a staged fashion: simple definition are first learned
and added to the background knowledge \cite{[Onsshot]}, allowing them to be reused during further learning tasks, and thus to build up more and more complex definitions. For instance, \textit{MIGO} would
first learn a simple definition of \textit{win\_1/1} for winning in
one move. Next, a predicate \textit{win\_2/1} describing the action
of winning in two moves can be built from \textit{win\_1/1} as shown
in Figure \ref{fig:intro}.

To evaluate performance we consider two evaluable games (Noughts and Crosses and Hexapawn). Our results demonstrate that substantially lower Cumulative Minimax Regret can be achieved by \textit{MIGO} compared to variants of reinforcement learning.

Our contributions are the introduction of a system for learning optimal two-player-game strategies (Section 3) and the description of its implementation (Section 4).  We demonstrate experimentally it converges faster than reinforcement learning systems and that learned strategies are transferable to more complex games (Section 5).
\section{Related Work}
\subsection{Learning game strategies}
Various early approaches to game strategies \cite{shapnib,Quinlan}
used the decision tree learner ID3 to classify minimax depth-of-win
for positions in chess end games.
These approaches used a set of carefully
selected board attributes as features. Conversely,
\textit{MIGO} is provided with a set of three relational primitives
({\em move/2}, {\em won/1}, {\em drawn/1}) representing the minimal
information a human would expect to know before playing a two
person game.
An ILP approach learned optimal chess endgame strategies at depth 0 or 1 \cite{Chess}. Examples are board positions taken from a database. Conversely, MIGO learns from game play.
\subsection{Reinforcement Learning}
Reinforcement Learning considers the task of identifying an optimal
agent policy to maximise the cumulative reward perceived by an agent. MENACE (Matchbox Educable Noughts And Crosses Engine) \cite{[MENACE]}
was the world's earliest reinforcement learning system and
was specifically designed to learn to play Noughts-and-Crosses.
An early manual version of MENACE used a stack of matchboxes,
one for each accessible board position.
Each box contained coloured beads representing possible moves. Moves
were selected by randomly drawing a bead from the current box.
After having completed a game, MENACE's punishment
or reward consisted of subtracting or adding beads according to
the outcome of the game.  This modified the probability
of the selected move being played in the position \cite{Brooks}. HER (Hexapawn
Educational Robot) \cite{Hexapawn} is a similar system for the game
of Hexapawn.

More generally, Q-learning \cite{QLearning} addresses the problem of
learning an optimal policy from delayed rewards and by trial and error.
The learned policy takes the form of Q-values for each actions available
from a state. A guarantee of asymptotic convergence to
optimal behaviour has been proved \cite{QLearning2}.

Deep Q-learning \cite{DQL} is an extension that uses a deep convolutional
neural network to approximate the different Q-values for each actions
given a particular state.
It provides better scalability which has been demonstrated through
a diverse range of tasks from the Atari 2600 games. However, this
framework generally requires the execution of many games to converge. Moreover,
the learned strategy is implicitly encoded into the Q-value parameters,
which do not provide interpretability. In \cite{Murray}, a hybrid
neural-symbolic system is described which address some of these drawbacks. A neural back-end transforms images into a symbolic representation and generates features. A symbolic front end performs action selection. Conversely, \textit{MIGO} is based upon a purely symbolic approach and the number of primitives considered is reduced.
\subsection{Relational Reinforcement Learning}
Relational reinforcement learning \cite{RRL} is a reinforcement learning
framework where states, actions and policies are represented relationally.
It benefits from background knowledge and declarative bias. It learns
a Q-function using a relational regression tree algorithm. Conversely,
the learning framework \textit{MIGO} is not based upon the identification
of Q-values but aims at deriving hypotheses describing an optimal
strategy. Relational reinforcement learning also provides the ability
to carry over the policies learned in simple domains to more complex
situations. However, most systems aim at learning single agent policy
and, in contrast to \textit{MIGO}, are not designed to learn
to play two person games.
\section{Theoretical Framework}
\subsection{Credit Assignment}
One can evaluate the success of a game by looking at its outcome.
However, a problem arises for assigning the reward to the various
moves performed. Reinforcement learning systems usually tackle this
so-called Credit Assignment Problem by adjusting parameter values
associated with the moves responsible for the reward observed. We
introduce theorems for identifying moves that are necessarily positive
examples for the task of winning and drawing.

We assume the learner $P_{1}$ plays against an opponent $P_{2}$ that follows an optimal strategy and that the game starts from a randomly chosen initial board $B$. We consider the following ordering over the different outcomes for $P_{1}$ and demonstrate the lemma below:
\begin{align*}
won \succ drawn \succ loss
\end{align*}
\begin{lemma} \label{lemma}
The expected outcome of $P_{1}$ can only decrease during
a game.
\end{lemma}
\begin{proof}
$P_{2}$ plays optimally and therefore any move of $P_{2}$ maintains or lowers the expected outcome. Therefore $P_{1}$ cannot increase its outcome.
\end{proof}
We demonstrate the Theorems below given these assumptions and Lemma \ref{lemma}:
\begin{theorem} \label{theorem1}
If the outcome is won for $P_{1}$, then every move
of $P_{1}$ is a positive example for the task of winning.
\end{theorem}
\begin{proof}
Suppose there exists a move of $P_{1}$ from the board $B_{1}$ to the board $B_{2}$ within the game sequence that is a negative example for the task of winning. Then the expected outcome of $B_{1}$ is won
and the expected outcome of $B_{2}$ is strictly lower with respect
to the order $\succ$. Then, following Lemma \ref{lemma} the outcome of the game is strictly lower than won, which leads to contradiction with the outcome observed.
\end{proof}
\begin{theorem} \label{theorem2}
We additionally assume an accurate strategy $S_{W}$ for winning has
been learned by the learner $P_{1}$. If the outcome of the game is drawn and if the execution of $S_{W}$ from B fails, then any move played by $P_{1}$ or $P_{2}$ is a positive example for the task of drawing.
\end{theorem}
\begin{proof}
The initial position does not have an expected outcome of won for $P_{2}$ otherwise the outcome would be won for $P_{2}$ since it plays optimally. The initial position is not an expected outcome of win for $P_{1}$ by assumption. Therefore, the expected outcome of B is drawn. It follows
from Lemma \ref{lemma} that every position reached during the game has an expected
outcome of drawn and that every move of both players is a positive
example for the task of drawing.
\end{proof}
Theorems \ref{theorem1} and \ref{theorem2} demonstrate that for an outcome of win for
$P_{1}$ or drawn and because the opponent plays optimally, the expected
outcome is necessarily maintained as won or drawn respectively. This
cannot be further generalised to $P_{2}$'s moves: an outcome of won
for $P_{2}$ might be the consequence of a mistake of $P_{1}$ who does
not play optimally.

One should also highlight the fact that Theorems \ref{theorem1} and \ref{theorem2} do
not provide any negative examples for \textit{win/2} or \textit{draw/2}, as these theorems do not help to evaluate moves for which the expected outcome
decreases. Practically, the learning system considered learns from
positive examples only.
\subsection{Game evaluation}
Given Theorems \ref{theorem1} and \ref{theorem2}, the opponent chosen is an optimal player
following the minimax algorithm. Both for Noughts-and-Crosses and
Hexapawn, and more generally for most fair two player games, the
opponent can always ensure a draw from the initial board, which leaves no opportunities for the learner to win. To ensure possibilities of winning, we start the game from a board randomly
sampled from the set of one move ahead accessible boards; this set provides different expected outcomes for the games considered. Then, the actual outcome
relies on both the initial board and the sequence of moves performed.
We define the minimax regret as follows:
\begin{definition}[Definition 3.4: ]
The \textit{minimax regret} of a game is the difference between the
minimax expected outcome of the initial board and the actual outcome
of the game.
\end{definition}
Practically, the minimax expected outcome of a board can be evaluated
from a minimax database computed beforehand. Definition 3.4 provides
an absolute measure to evaluate the performances of a learning algorithm
as it does not rely on the choice of initial board. Thereafter, we
evaluate the cumulative minimax regret to compare different learning
systems.
\subsection{Meta-Interpretive Learning (MIL)}
The system \textit{MIGO} introduced in this work is a MIL system. MIL is a form of ILP~\cite{[MIL],[MIL2]}.
The learner is given a set
of examples $E$ and background knowledge $B$ composed of a set of
Prolog definitions $B_{p}$ and metarules $M$ such that $B = B_{p} \cup M$. The aim is to generate a hypothesis $H$ such that $B,H \models E$.
The proof is based upon an adapted Prolog meta-interpreter. It first
attempts to prove the examples considered deductively. Failing this,
 it unifies the head of a metarule with the goal, and saves the resulting
meta-substitution. The body and then the other examples are similarly
proved. The meta-substitutions recovered for each successful proofs
are saved and can be used in further proofs by substituting them
into their corresponding metarules. Key features of MIL are that it supports predicate invention, the
learning of recursive programs and Abstraction \cite{[Abstraction]}.
In the following, we use the MIL system \textit{Metagol} \cite{metagol}.
\subsection{MIGO algorithm}
We present within this section details of the MIGO algorithm.
\paragraph{Learning from positive examples}
Theorems \ref{theorem1} and \ref{theorem2} provide a way of assigning
positive labels
to moves. Therefore, the learning is based upon positive
examples only. This is possible because of Metagol's strong language
bias and ability to generalise from a few examples only. However,
one pitfall is the risk of over-generalisation due to the absence
of negative examples.
\paragraph{Dependent Learning}
For successive values of $k$ a series of inter-related definitions are learned for predicates $\mbox{win}\_k(A,B)$ and $\mbox{draw}\_k(A,B)$. These predicates define maintenance of minimax win and draw in $k$-ply when moving from position $A$ to $B$. The learning algorithm is presented as Algorithm \ref{alg:learningprotocol}, each action 'learn' represents a call to Metagol.
This approach is related to
Dependent Learning \cite{[Onsshot]}.
The idea is to first learn low-level predicates.
They are derived from single examples with limited complexity.
The definitions are added into the background knowledge such that
they can be used in further definitions. The process iterates
until no further predicates can be learned.

\begin{algorithm}[tb]
\caption{MIGO Algorithm}
\label{alg:learningprotocol}
\footnotesize
\textbf{Input}: Positive examples for win$\_$k and draw$\_$k\\
\textbf{Output}: Strategy for win$\_$k and draw$\_$k
\begin{algorithmic}[1]
\FOR{k in [1,Depth]}
\FOR{each example of win$\_$k/2}
\STATE one shot learn a rule and add it to the BK
\ENDFOR
\STATE Learn win$\_$k/2 and add it to the BK
\ENDFOR
\FOR{k in [1,Depth]}
\FOR{each example of draw$\_$k/2}
\STATE one shot learn a rule and add it to the BK
\ENDFOR
\STATE Learn draw$\_$k/2 and add it to the BK
\ENDFOR
\end{algorithmic}
\end{algorithm}
\normalsize
\paragraph{Mixed Learning and Separated Learning}
Theorem \ref{theorem2} assigns positive labels to \textit{draw/2}
examples assuming a winning strategy $S_{W}$ has already been learned.
In practice, we distinguish two variants of \textit{MIGO}:
    \vspace{-1mm}
\begin{enumerate}[leftmargin=*]
\item Separated Learning: \textit{win/2} and \textit{draw/2} are learned
in two stages. \textit{Win/2} is first learned. 
When a strategy for \textit{win/2} is stable for a given
number of iterations the learner starts learning \textit{draw/2}.
    \vspace{-1mm}
\item Mixed Learning: \textit{win/2} and \textit{draw/2} are learned
simultaneously. Examples of \textit{draw/2} are first evaluated with
the current strategy for \textit{win/2}. If this latter is updated,
 examples of \textit{draw/2} are re-tested against the new version
of \textit{win/2}.
\end{enumerate}
\section{Implementation}
\subsection{Representation}
A board $B$ is encoded as a 9-vector of marks from the set $\{O,X,Empty\}$.
States $s(B,M)$ are atoms that represent the current
board $B$ and the active player $M$.
\subsection{Primitives and Metarules}
The language belongs to the language class $H_{2}^{2}$, which is the subset of Datalog logic programs with predicates of arity at most 2 and at most 2 literals in the body of each clause.
Learned programs are formed of dyadic predicates, representing actions,
and monadic predicates, representing fluents. The background knowledge
contains a general move generator \textit{move/2}, which is an action
that modifies a state $s(B,M)$ by executing a move on board $B$ and
updating the active player $M$. \textit{Move/2} only holds for valid
moves; in other words, the learner already knows the rules of the
game. The background knowledge also contains two fluents: a won classifier
\textit{won/1} and a drawn classifier \textit{drawn/1}. They hold
when a board is respectively won or drawn.

We consider the metarules \textit{postcond} and \textit{negation}
described in Table \ref{tab:metarules}. The metarule \textit{negation}
expresses the logical negation for primitive predicates, and is implemented
as negation as failure. This form of Negation does not introduce
invented predicates in Metagol.
\begin{table}
\centering
\begin{adjustbox}{width=0.35\textwidth}
\begin{tabular}{cr}  
\toprule
Name  & Metarule\\
\midrule
\emph{postcond} & $P(A,B) \leftarrow Q(A,B),R(B).$\\
\emph{negation} & $P(A,B) \leftarrow Q(A,B), not(R(B,C)).$\\
\bottomrule
\end{tabular}
\end{adjustbox}
\caption{Metarules considered: the letters P,Q and R denote existentially quantified higher order variables. The letters A, B and C denote universally quantified first-order variables.}
    \vspace{-4mm}
\label{tab:metarules}
\end{table}
\subsection{Execution of the strategy}
For each rule learned for \textit{win$\_$i} and \textit{draw$\_$i}
a clause of the form below is added to the background knowledge:
\small    \vspace{-1mm}
\begin{verbatim}
win(A,B) :- win_i(A,B).
draw(A,B) :- draw_i(A,B).
\end{verbatim}
\normalsize
When executing a strategy described with a hypothesis $H$, the move
performed is the first one consistent with $H$. Practically, it first
attempts to prove \textit{win$\_{i}$/2} for increasing values for $i$.
Failing that, it attempts to prove \textit{draw$\_{i}$/2} for increasing
values for $i$. If these proofs fail, a move is selected at random
among the possible moves.\\
The opponent plays a deterministic
minimax strategy that yields the best outcome in the minimum number
of moves. 
\subsection{Learning a strategy}
At the end of a game, the outcome is observed and the sequence of
visited boards is divided into moves. The depth of each board is measured as the number of full moves until the end of the game in the observed sequence. Moves are added to the set of
positive examples for win$\_k$/2 or draw$\_k$/2 if they satisfy Theorems \ref{theorem1} or \ref{theorem2}.
Strategies are relearned from scratch after each game using the
MIGO algorithm presented above. One additional constraint is
added such that \textit{draw/2} cannot be learned before \textit{win/2}
since this would cause the learner to always draw and never win.
\section{Experiments}
\subsection{Experimental Hypothesis}
This section describes experiments which evaluate the performance of \textit{MIGO} for the task of learning optimal two player game strategies\footnote{Code for these experiments available at\\ \url{https://github.com/migo19/migo.git} }. We use the games of Noughts-and-Crosses and a variant of the game of Hexapawn \cite{Hexapawn}. \textit{MIGO} is compared against the reinforcement learning systems MENACE / HER, Q-learning and Deep Q-learning. Accordingly, we investigate the following null hypotheses:
\medbreak \noindent
\textit{\textbf{Null Hypothesis 1: }} \textit{\textit{MIGO} cannot converge faster than MENACE / HER, Q-learning and Deep Q-learning for learning optimal two-player game strategies.}
\smallbreak
We additionally test the ability of \textit{MIGO} to transfer learned strategies to more complex games, and thus verify the following null hypothesis:
\medbreak  \noindent
\textit{\textbf{Null Hypothesis 2: }} \textit{\textit{MIGO} cannot transfer the knowledge learned during a previous task to a more complex game.}
\subsection{Convergence}
\subsubsection{Materials and Methods}
\begin{figure}
\centering
\captionsetup{justification=centering}
\subfloat{\includegraphics[width=0.06\textwidth]{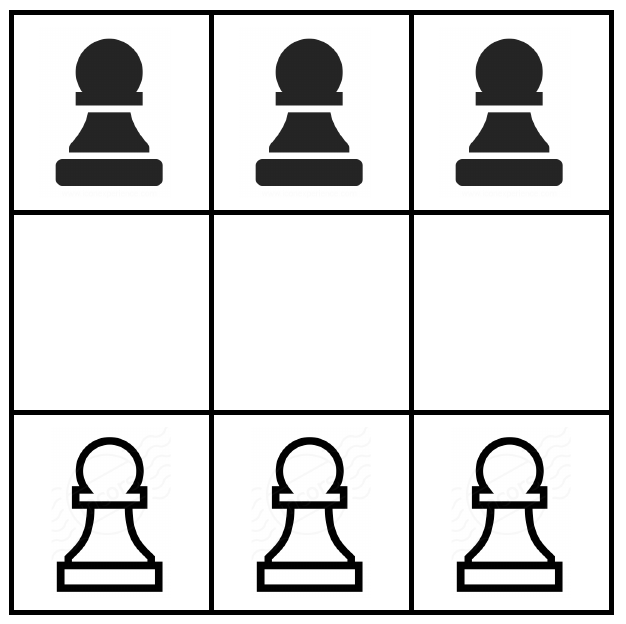}} \quad
\subfloat{\includegraphics[width=0.07\textwidth]{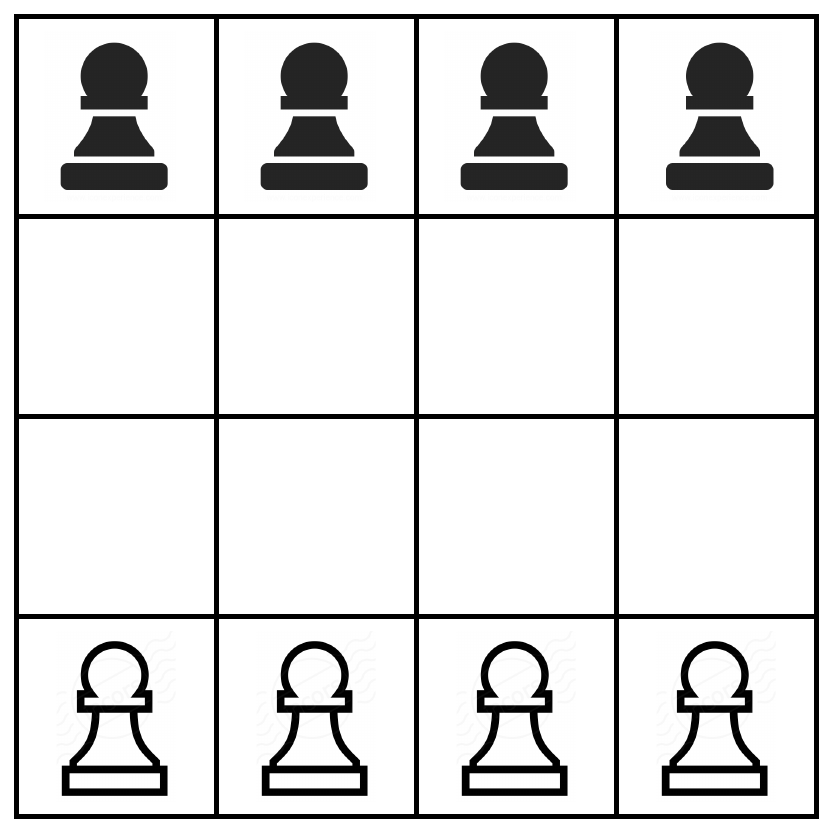}}
    \caption{Initial boards for Hexapawn$_{3}$ and Hexapawn$_{4}$}
   \label{fig:hexapawn}
    \vspace{-3mm}
\end{figure}
\begin{table}
\centering
\begin{adjustbox}{width=0.4\textwidth}
\begin{tabular}{r|r|r|r}  
\toprule
& OX & Hexapawn$_{3}$ & Hexapawn$_{4}$ \\
\midrule
MIGO mixed learning & 1.5$.10^{-1}$ & 3.0$.10^{-3}$& 3.9 \\
MIGO separated learning & 8.9$.10^{-2}$ & 2.8$.10^{-3}$ & 3.8\\
MENACE / HER & 1.5$.10^{-3}$ & 2.7$.10^{-4}$& /\\
Q-Learning & 2.3$.10^{-1}$  & 1.9 $.10^{-3}$ & 2.7 $.10^{-1}$\\
Deep Q-Learning & 2.4$.10^{-1}$  & 1.7$.10^{-2}$ & 2.1 $.10^{-1}$\\
\bottomrule
\end{tabular}
\end{adjustbox}
\captionsetup{justification=centering}
\caption{Average CPU time (seconds) of one iteration}
   \label{fig:time}
      \vspace{-3mm}
\end{table}
\paragraph{Common}
We provide \textit{MIGO}, Menace / HER and Q-learning with the same set of initial boards randomly sampled from the set of one-full-move-ahead positions - positions that result from one move of each player. The systems studied play games starting from these initial boards, and they face the same deterministic minimax player. Therefore, the only variable in the experiments is the learning system. It is assumed the learner always starts the game. The performance is evaluated in terms of cumulative minimax regret.\\
We follow an implementation of Tabular Q-learning available from \cite{Qlearningcode} and used the parameter values which were provided for the Q-learning algorithm: the exploration rate is set to 0; the initial q-values are $1$; the discount factor is $\gamma = 0.9$ and the learning rate $\alpha = 0.3$.\\
Similarly, we follow an implementation of Deep Q-learning available from \cite{DeepQlearningcode}. The provided parameters were used: the discount factor is set to $0.8$; the regularization strength to $0.01$; the target network update rate to $0.01$; the initial and final exploration rate are $0.6$ and $0.1$ respectively.\\
The results presented here have been averaged oven 40 runs for Hexapawn$_{3}$ and 20 for Noughts and Crosses. Average running times are presented in Figure \ref{fig:time}.
\paragraph{Noughts-and-Crosses}
The set of initial boards comprises 12 boards taking into account rotations and symmetries of the board. Among them 7 are expected win, and 5 are expected draw. Therefore the worst case regret of a random player is 1.58. The counter for starting learning \textit{draw/2} is set to 10.
\paragraph{Hexapawn}
Hexapawn's initial board is represented in Figure \ref{fig:hexapawn}. The goal of each player is to advance one of their pawns to the opposite end. Pawns can move one square forward if the next square is empty or capture another pawn one square diagonally ahead of it \cite{Hexapawn}. Rules have been modified: the game is said to be drawn when the current player has no legal move.Thereafter, we refer to Hexapawn$_{3}$ and Hexapawn$_{4}$ for the game of Hexapawn in dimensions 3 by 3 and 4 by 4 respectively. The set of initial boards comprises 5 boards taking into account the vertical symmetry. Among them, 3 are expected draw and 2 are expected win. Therefore the average worst case regret is 1.4. As the dimensions are smaller for Hexapawn$_{3}$ than for Noughts and Crosses, the counter for starting learning \textit{draw/2} is set to 5.
\subsubsection{Results}
\begin{figure}
\centering
\captionsetup{justification=centering}
\subfloat[Noughts-and-Crosses]{\includegraphics[width=0.41\textwidth]{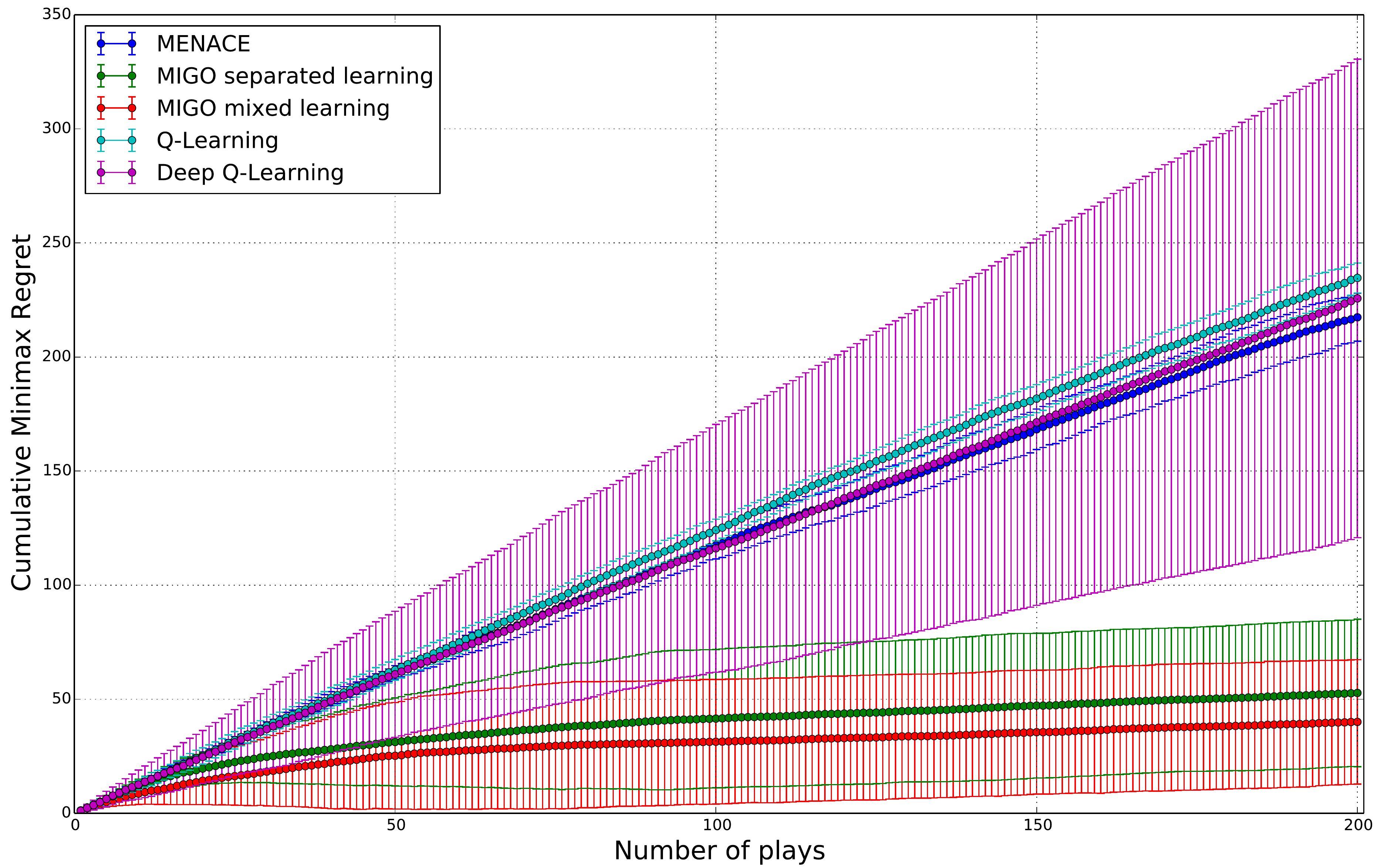}} \\
    \vspace{-3mm}
\subfloat[Hexapawn$_{3}$]{{\includegraphics[width=0.41\textwidth]{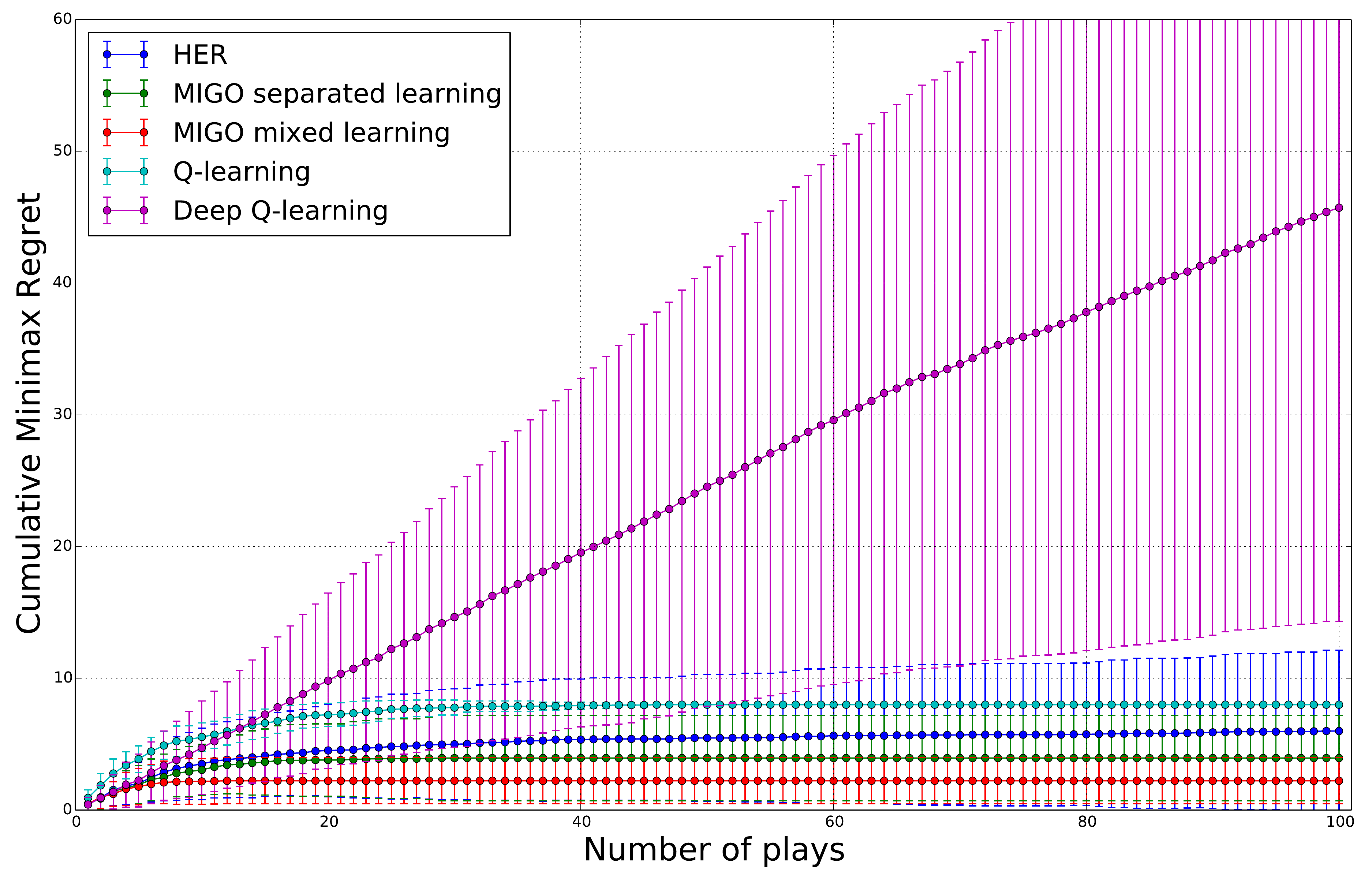}}}
    \vspace{-2mm}
    \caption{Cumulative Minimax Regret versus the number of iterations for Noughts-and-Crosses and Hexapawn$_{3}$}
   \label{fig:samplecomp}
   \vspace{-3mm}
\end{figure}
Results are presented in Figure \ref{fig:samplecomp} and show that \textit{MIGO} converges faster than MENACE / HER, Q-learning and Deep Q-learning for both games, refuting null hypothesis 1. As the maximum depth is larger for Noughts-and-Crosses than for Hexapawn$_{3}$, all systems require more iterations to converge. Deep Q-learning performs worst for Hexapawn$_{3}$ as the parameters selected are the ones tuned for Noughts and Crosses and might not be adapted. For both games, mixed learning has lower cumulative regret than separated  learning, because mixed learning does not waste any examples of \textit{draw/2} from the initial period in which \textit{win/2} is being learned and it does not stop learning \textit{win/2} after the initial period.

Rules learned by \textit{MIGO} are presented in Figure \ref{fig:rules}. \textit{MIGO} converges toward this full set of rules when playing Noughts-and-Crosses. Because the maximum depth of Hexapawn$_{3}$ is 2, \textit{MIGO} learns up to the double line when playing Hexapawn$_{3}$. If unfolding, the first rule can be translated into English as: \textit{State A is won at depth 1 if there exists a move from A to B such that B is won.} Similarly, winning at depth 2 can be described with the following statement: \textit{State A is won at depth 2 if there exists a move of the current player from A to B such that B is not immediately won for the opponent and such that the opponent cannot make a move from B to C to prevent the current player from immediately winning.} This statement is similar to the one presented in section 1. Finally, winning at depth 3 can be explained as: \emph{State A is won at depth 3 for the current player if there exists a move from A to B such that B is not won for the opponent in 1 or 2 moves and such that the opponent cannot make a move from B to C to prevent the current player from winning in 1 or 2 moves.} None of the other systems studied can provide similar explanation about the moves chosen. Rules are built on top on each other, the calling diagram in Figure \ref{diagram} represents the dependencies between each learned predicates.
\begin{table}
\centering
\begin{adjustbox}{width=0.45\textwidth}
\begin{tabular}{|c|c|}  
\toprule
Depth & Rule \\ \midrule
1 & \texttt{win\_1(A,B):-win\_1\_1\_1(A,B),won(B).} \\
& \texttt{win\_1\_1\_1(A,B):-move(A,B),won(B).}\\ \hline
& \texttt{draw\_1(A,B):-draw\_1\_1\_3(A,B),not(win\_1(B,C)).}\\
& \texttt{draw\_1\_1\_3(A,B):-move(A,B),not(win\_1(B,C)).}\\ \hline
2 & \texttt{win\_2(A,B):-win\_2\_1\_1(A,B),not(win\_2\_1\_1(B,C)).}\\
& \texttt{win\_2\_1\_1(A,B):-move(A,B),not(win\_1(B,C)).}\\ \hline
& \texttt{draw\_2(A,B):-draw\_2\_1\_1(A,B),not(win\_1(B,C)).}\\
& \texttt{draw\_2\_1\_1(A,B):-draw\_1(A,B),not(win\_1(B,C)).}\\ \hline \hline
3 & \texttt{win\_3(A,B):-win\_3\_1\_1(A,B),not(win\_3\_1\_1(B,C)).} \\
& \texttt{win\_3\_1\_1(A,B):-win\_2\_1\_1(A,B),not(win\_2(B,C)).} \\ \hline
& \texttt{draw\_3(A,B):-draw\_3\_1\_10(A,B),not(draw\_1\_1\_12(B,C))}.\\
& \texttt{draw\_3\_1\_10(A,B):-draw\_2(A,B),not(draw\_1\_1\_12(B,C))}.\\ \hline
4 & \texttt{draw\_4(A,B):-draw\_4\_1\_2(A,B),not(draw\_1\_1\_12(B,C))}.\\
& \texttt{draw\_4\_1\_2(A,B):-draw\_3(A,B),not(draw\_1\_1\_12(B,C))}.\\
\bottomrule
\end{tabular}
\end{adjustbox}
\captionsetup{justification=centering}
    \vspace{-1mm}
\caption{Example of rules learned for Noughts-and-Crosses (all) and Hexapawn$_{3}$ (above the double line)}
    \vspace{-4mm}
   \label{fig:rules}
\end{table}
\begin{figure}[t]
\captionsetup{justification=centering}
\centering
\subfloat{\includegraphics[width=0.35\textwidth]{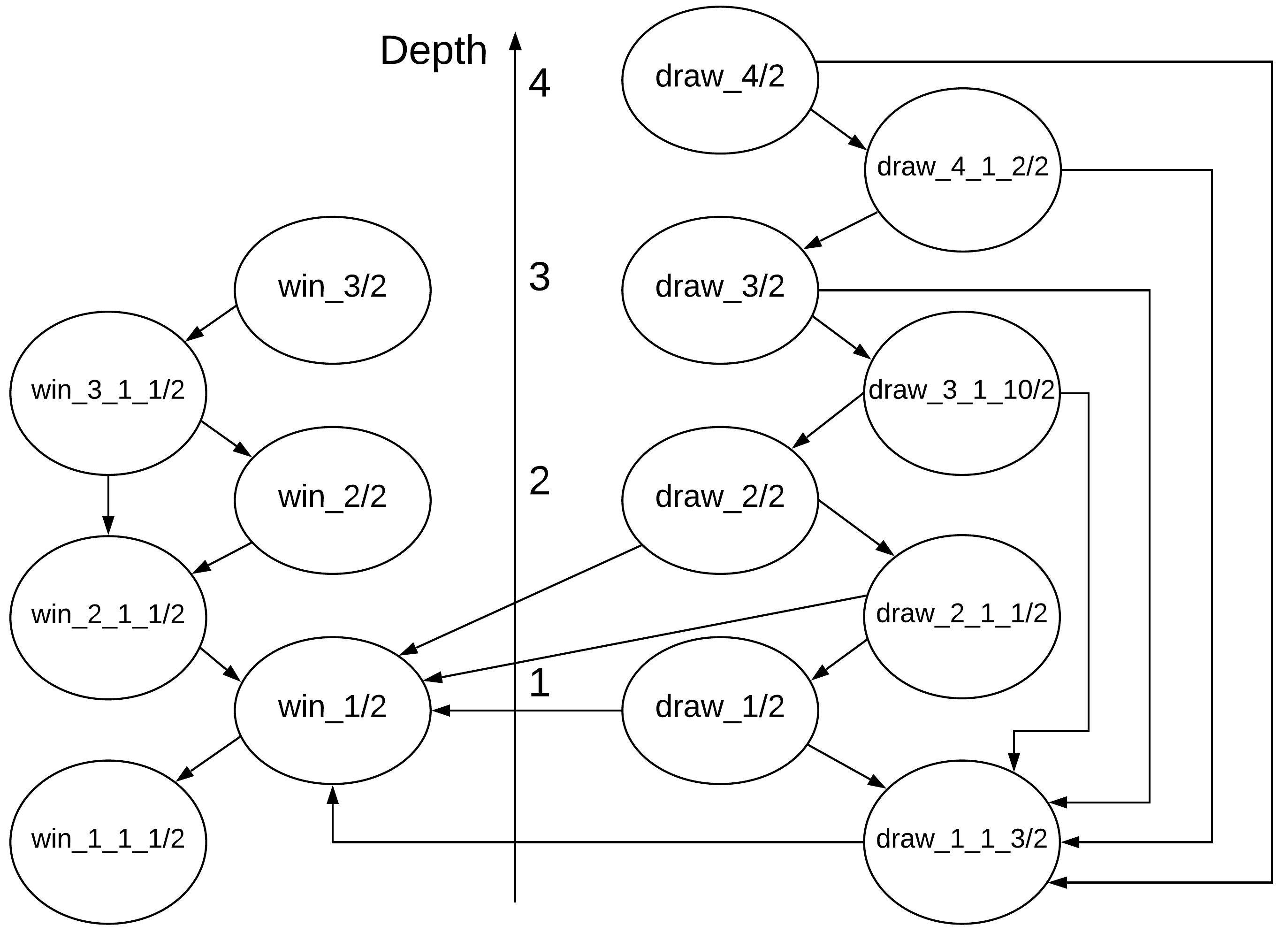}}
\captionsetup{justification=centering}
    \vspace{-1mm}
\caption{Calling diagram of Learned Strategies}
    \vspace{-4mm}
\label{diagram}
\end{figure}
\subsubsection{Discussion}
MENACE, HER and Q-learning encode the knowledge into the parameters (number of beads or Q-values). The states and their parameters are unique for each board. This results in a weaker generalisation ability: knowledge cannot be transferred from one state to another. Deep Q-learning can provide some generalisation ability; however, it is only visible after a large number of iterations. Conversely, \textit{MIGO} generalises the boards characteristics and each rule learned describes a set of states, which considerably reduces the number of parameters to learn and therefore the number of examples required.

The reinforcement learning systems tested have an implicit representation of the problem. For instance, no geometrical concepts have been encoded. Conversely, \textit{MIGO} benefits from a background knowledge which describes the notion of winning, and from which it can extract a notion of alignment. This allows a degree of explanation.

The running time increases rapidly with the state dimensions for \textit{MIGO}. This reflects the increasing execution time of the learned strategy which is not efficient since a deep evaluation requires extensive evaluation to decide whether a move leads to a win.

MENACE / HER are specifically tailored for theses games. Conversely, Q-learning and Deep Q-learning are a general approaches that can tackle a wide range of tasks, providing that parameters are tuned. \textit{MIGO} benefits from underlying assumptions which reduce its range of applications. However, the primitives are abstract enough to allow playing a wide range of games and support transferring knowledge from one game to another as we will demonstrate in the next section.

\subsection{Transferability}
\begin{figure}[t]
\centering
\captionsetup{justification=centering}
\subfloat[Hexapawn$_{3}$ to Noughts and Crosses]{\includegraphics[width=0.41\textwidth]{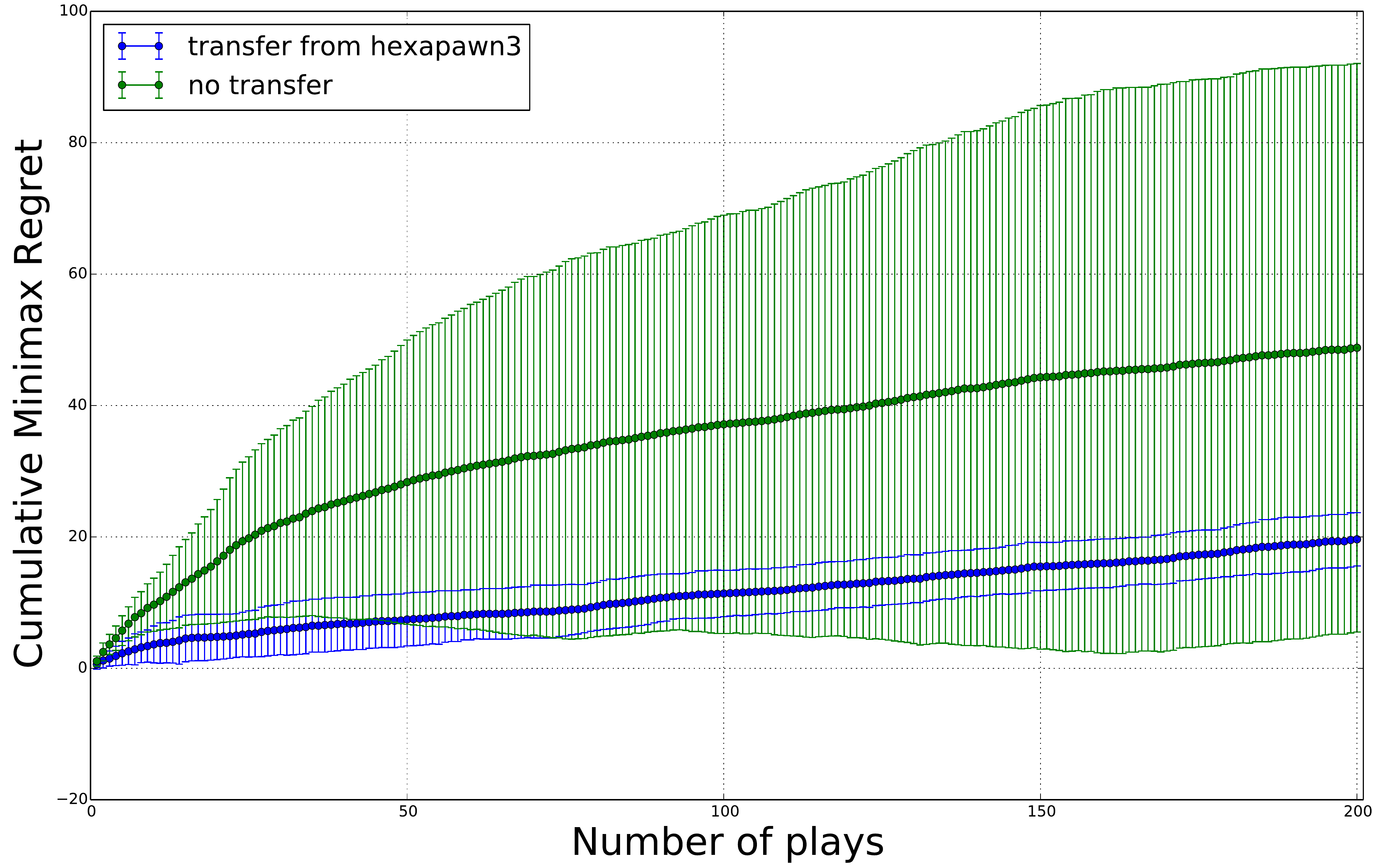}}\\
    \vspace{-2mm}
\subfloat[Noughts and Crosses to Hexapawn$_{4}$. Similar results are obtained from Hexapawn$_{3}$ to Hexapawn$_{4}$.]{\includegraphics[width=0.41\textwidth]{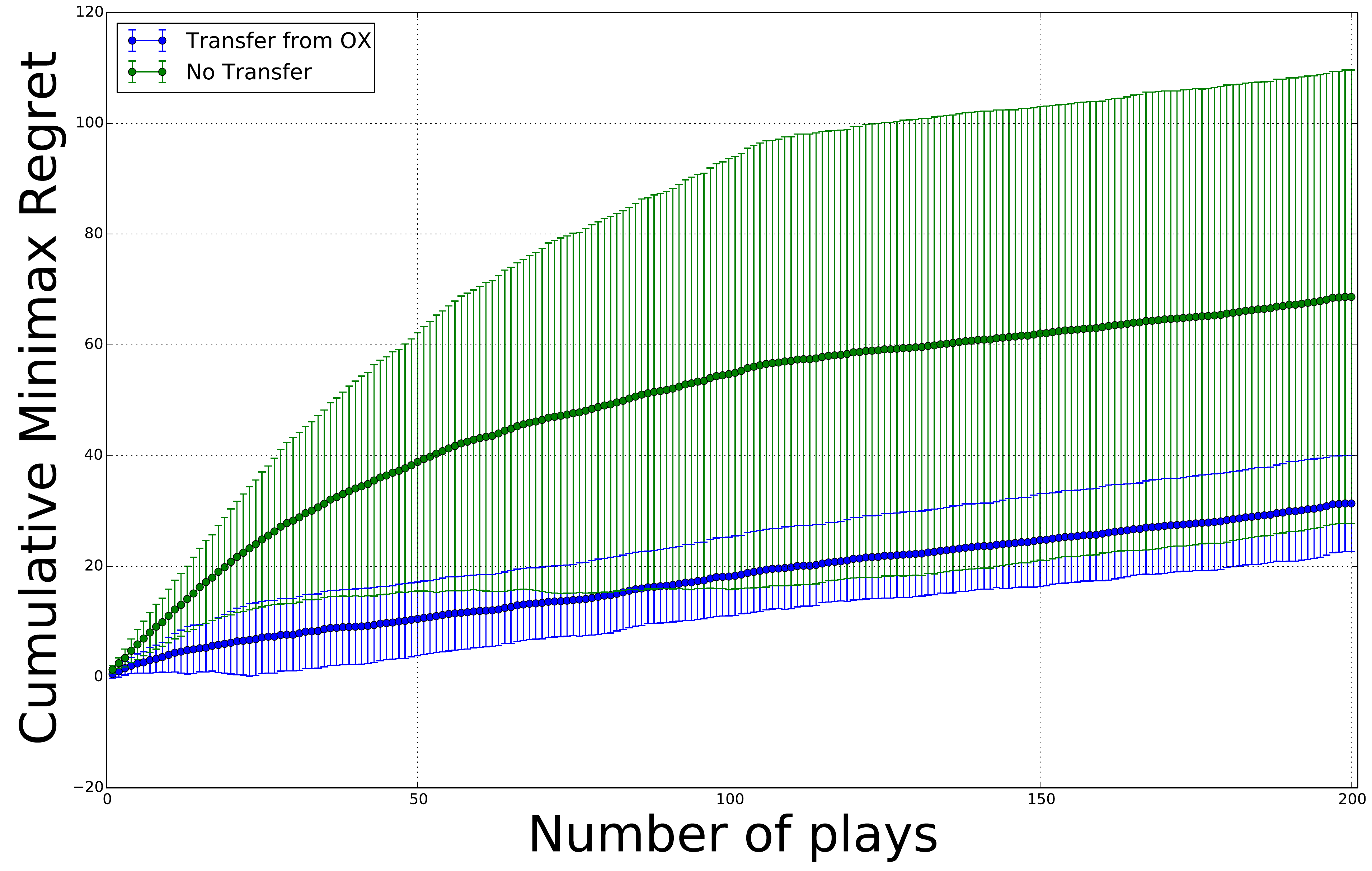}}
    \vspace{-2mm}
\caption{Transfer Learning}
\label{fig:transfer}
\vspace*{-4mm}
\end{figure}
\paragraph{Materials and Methods}
Strategies are first learned for Hexapawn$_{3}$ and Noughts-and-Crosses respectively. Strategies are learned with mixed learning and for 100 iterations for Hexapawn$_{3}$ and 200 iterations for Noughts and Crosses. The resulting learned program is transferred to the next learning task, which is learning a strategy for Hexapawn$_{4}$. Results have been averaged over 20 runs.
\paragraph{Results}
The results presented in Figure \ref{fig:transfer} show that transferring the knowledge learned in a previous task help to converge faster, thus refuting null hypothesis 2. Since the learner benefits from an initial knowledge, it is substantially improved compared to an initial random player.
\section{Conclusion and Future Work}
This article introduces a novel logical system named \textit{MIGO} for learning two-player-game strategies. It is based upon the MIL framework. This system distinguishes itself from classical reinforcement learning by the way that it addresses the Credit Assignment Problem. Our experiments have demonstrated that \textit{MIGO} achieves lower Cumulative Minimax Regret compared to Deep and classical Q-Learning. Moreover, we have demonstrated that strategies learned with \textit{MIGO} are transferable to more complex games. Strategies have also been shown to be relatively easy to comprehend.
\paragraph{Future Work}
One limitation of the system presented is the risk of over-generalisation, observable in the strategy learned. We will further extend the implementation to include a more thorough context for learning from positive examples such as the one presented in \cite{[Positive]}.

The running time suggests that the execution time of the learned strategies increases with the dimensions of the states, which limits scalability. We will further extend \textit{MIGO} to optimise the execution time for hypothesised programs. Selection of hypotheses could be performed following the idea described in \cite{[Metaopt]}.

Another limitation to scalability is the restriction imposed by the initial assumptions. The current version of \textit{MIGO} requires an optimal opponent, which is intractable in large dimensions. We will further extend this system by relaxing Theorems \ref{theorem1} and \ref{theorem2} and weakening the optimal opponent assumption. A solution could be to learn from self-play.

Because \textit{MIGO} benefits from a strong declarative bias, the sample complexity is much improved compared to other approaches. However, most of the examples are wasted as no labels could be attributed. We plan to evaluate whether Active Learning could further help to reduce the sample complexity. The learner could choose an initial board to start the game, the choice being based upon an information gain criterion.

Although learned strategies provide a certain form of explanation, we will further study how comprehensible learned strategies are. We will evaluate whether MIGO can fullfill Michie's Machine Learning Ultra Strong criterion, which requires the learner to be able to teach the learned hypothesis to a human \cite{UltraStrongML}.

Despite these limitations, we believe the novel system introduced in this work opens exciting new avenues for machine learning game strategies.

\bibliographystyle{named}
\bibliography{biblio}

\end{document}